\documentclass[english]{article}
\usepackage[T1]{fontenc}
\usepackage[latin9]{inputenc}
\usepackage{geometry}
\usepackage{float}
\usepackage{amsmath}
\usepackage{amsthm}
\usepackage{amssymb}
\usepackage{graphicx}
\usepackage{esint}

\makeatletter

\providecommand{\tabularnewline}{\\}

\numberwithin{equation}{section}
\numberwithin{figure}{section}
\theoremstyle{plain}
\newtheorem{thm}{\protect\theoremname}
\theoremstyle{definition}
\newtheorem{problem}[thm]{\protect\problemname}
\theoremstyle{remark}
\newtheorem{rem}[thm]{\protect\remarkname}
\theoremstyle{plain}
\newtheorem{prop}[thm]{\protect\propositionname}

\makeatother

\usepackage{babel}
\providecommand{\problemname}{Problem}
\providecommand{\propositionname}{Proposition}
\providecommand{\remarkname}{Remark}
\providecommand{\theoremname}{Theorem}

\begin{document}
\title{Probabilistic Approach for Detection of High-Frequency Periodic Signals
using an Event Camera}
\author{David El-Chai Ben-Ezra, Ron Arad, Ayelet Padowicz, \\
Israel Tugendhaft}
\maketitle
\begin{abstract}
Being inspired by the biological eye, event camera is a novel asynchronous
technology that pose a paradigm shift in acquisition of visual information.
This paradigm enables event cameras to capture pixel-size fast motions
much more naturally compared to classical cameras. 

In this paper we present a new asynchronous event-driven algorithm
for detection of high-frequency pixel-size periodic signals using
an event camera. Development of such new algorithms, to efficiently
process the asynchronous information of event cameras, is essential
and being a main challenge in the research community, in order to
utilize its special properties and potential.

It turns out that this algorithm, that was developed in order to satisfy
the new paradigm, is related to an untreated theoretical problem in
probability: let $0\leq\tau_{1}\leq\tau_{2}\leq\cdots\leq\tau_{m}\leq1$,
originated from an unknown distribution. Let also $\epsilon,\delta\in\mathbb{R}$,
and $d\in\mathbb{N}$. What can be said about the probability $\Phi(m,d)$
of having more than $d$ adjacent $\tau_{i}$-s pairs that the distance
between them is $\delta$, up to an error $\epsilon$ ? This problem,
that reminds the area of order statistic, shows how the new visualization
paradigm is also an opportunity to develop new areas and problems
in mathematics. 
\end{abstract}

\section{Introduction}

Event camera is a bio-inspired sensor that does not give information
about its whole field of view, but only about changes in it. Each
pixel of an event camera is asynchronously independent and responds
only when it feels a brightness change that reaches a predefined threshold.
Hence, in contrast to standard cameras that sample the whole field
of view in a certain sampling-rate and output a sequence of synchronous
frames, event cameras provide data only when a certain pixel feels
a predifined change in the illumination, and the output is a list
of asynchronous polarized \textquotedbl events\textquotedbl{} sorted
by their timestamp. Just to make it clear: idealy, if nothing changes
in the field of view, the output of an event camera will be empty,
while the amount of data of a frame-based camera is independent with
changes in the field of view. Event cameras come with some special
properties: timestamp in resolution of microseconds, low latency,
high dynamic range (over 120 dB), and low power consumption (see \cite{key-1,key-2}
for review). 

Regarding missions that require detction of pixel-size fast motions,
event cameras enjoy a system engineering built-in advantage compared
to frame-based camera: while classical cameras need to sample the
whole field of view with high frame-rate and process a lot of redundant
data in order to find the signal, event cameras just ``wait'' for
it, and keep only the relevant information from the signal. For event
camera, it does not matter if the motion is fast or slow, when it
comes, the camera will respond to it. Hence, in such missions, the
new paradigm can naturally be used to surpass the performance of frame-based
cameras.

A classical mission of that type is the detection of high-frequency
periodic signals. Going back to frame-based cameras, the common approach
to distinguish between signals of a given frequency and random flickering
is based on the use of Fourier transform. Using frame cameras, one
needs to sample the signal at a rate higher than double the desired
frequency (due to the Nyquist criterion). This approach works nicely
when the desired signal is not too fast, and not too short. However,
if the frequency of the signal is higher than $\sim$1 kHz, then sampling
it with a frame-based camera and analyzing it with this approach can
become quite cumbersome, and if the signal is too short, say 1 $\mu s$
long, one might miss it altogether, or at least a large portion of
it, due to the inherent camera dead time. We got so used to this approach
that it takes a bit of thinking in order to realize that using the
deep theory of Fourier transform for this mission sounds like using
a 5-kilo hammer in order to knock a nail. 

The asynchronous bio-inspired paradigm of event cameras offers an
approach that sounds much more natural and intuitive: to check whether
the time difference between adjacent events in a certain pixel corresponds
to the temporal-period of the desired frequency. This approach is
inherently more intuitive as it is the way human vision works to detect
repetitive signals. Adapting this approach to an event camera, one
can easily surpass the 1 kHz limit, without missing short signals,
as event-cameras do not have a dead time. This approach was implemented
in \cite{key-3} to show the potential of event cameras to track led
markers blinking at high frequency ($>$ 1 kHz) carried by a drone.
For more applications of this approach see \cite{key-4,key-5}. For
event driven Fourier transform approach see \cite{key-6}. 

As event cameras work differently from frame-based cameras, a main
challenge in unlocking their potential is to develop novel asynchronous
event driven methods and algorithms to process their output (see \cite{key-7,key-8,key-9,key-10,key-11,key-12,key-13,key-14}
for some examples). In this paper, we use the notion \textquotedbl time-surface\textquotedbl{}
and present a new asynchronous event-based algorithm to distinguish
between signals of a given frequency and random flickering, based
on the aforementioned intuitive approach (see \cite{key-15,key-16,key-17,key-18,key-19,key-20}
and \cite{key-1} for more applications of the notion ``time surface''). 

Regarding the work in \cite{key-3}, the algorithm and analysis presented
in this paper have a few advantages. The first is that using the method
in this paper, one can approximate the probability for false alarm,
something that was not dealt in \cite{key-3}, but is important for
real world applications. This emphasises the importance of the presented
probabilistic approach. The second is the simplicity of the method.
With the method presented here, one can consider events of a single
polarity. Specifically, we consider here only positive events. This
is an advantage as in certain situations and for some types of event
cameras, it is not always easy in practice to tune the thresholds
of the camera to achieve good accuracy for both polarities. Regarding
its decision making, another adavntage is that the algorithm does
not go over each pixel and check the signal. The detections list is
generated incidentally through looping over the events. In other words,
only the temporal dimension plays role in running the algorithm and
locating the signal. 

As mentioned in the abstract, the analysis of the algorithm boils
down naturally to the theoretical interesting probabilistic problem
of estimating the probability

\[
\Phi(m,d)=P(\#\{j\,\,|\,\,|\tau_{j+1}-\tau_{j}-\delta|<\epsilon\}\geq d)
\]
where $0\leq\tau_{1}\leq\tau_{2}\leq\cdots\leq\tau_{m}\leq1$ are
originated by an unknown distribution, $\epsilon,\delta\in\mathbb{R}$,
and $d\in\mathbb{N}$. Quite surprisingly, as far as we know, this
problem was not studied in the literature. We say that it is surprising,
as the problem seems to be natural to be asked under the area of order
statistic. This shows the potential of the new paradigm to develop
new areas in pure mathematics.

In the paper, we present an analysis of the problem that can be practical
in certain cases. However, an accurate treatment of the problem stays
open, even in the simple case where the distribution is uniform. Our
analysis uses basic tools in probability theory to approximate the
function $\Phi(m,d)$ by the explicit expression
\begin{align*}
 & Q(m,d)=\sum_{l=d}^{m-1}\binom{m-1}{l}P(m)^{l}(1-P(m))^{m-1-l}\\
\mathrm{where} & \,\,\,\,\,\,P(m)=T^{m+1}\left(\frac{1}{(T+\delta-\epsilon)^{m+1}}-\frac{1}{(T+\delta+\epsilon)^{m+1}}\right).
\end{align*}

\begin{problem}
Is there any explicit formula for $\Phi(m,d)$ when the distribution
is uniform, or under any other assumptions on the distribution? 
\end{problem}

\begin{problem}
Estimate the difference between $\Phi(m,d)$ under certain assumptions
on the distribution, and $Q(m,d)$, or any other explicit approxomation
of $\Phi(m,d)$.
\end{problem}

At the end of the paper, we demonstrate the algorithm performance,
using the presented analysis, and show how its decision making distinguishes
between the periodic signals of streetlights flickering at 100 Hz,
and other random signals during twilight, when many objects in the
field of view are flickering as a result of sun glittering. This example
inspires to use the idea behind the presented algorithm, not only
for detection, but also for flicker removal (see \cite{key-21} for
a different method).

\section{The Output of an Event Camera}

Contrary to a frame-based camera, the output of an event camera does
not consist of a synchronous series of matrices that describe the
gray level of the pixels. Instead, it consists of an asynchronous
list of ``events'' that are generated in the following way. 

Any pixel of the camera ``remembers'' a certian reference value
for the intensity of the light in the pixel. Then, the pixel measures
continuously the change of the intensity of the light with relation
to the reference value. When the intensity of the light changes enough,
and the change reaches a predefined value, the pixel takes two actions:
\begin{enumerate}
\item It updates the reference value to the current intensity value.
\item Reports an ``event'' that contains the following information: 
\begin{enumerate}
\item The timestamp of the change, in reslution of microseconds.
\item The coordinates of the pixel.
\item The polarity of the event, namely, if the event was triggered by a
positive change of the light or a negative one. 
\end{enumerate}
\end{enumerate}
The list of the events, which is the output of the camera, is given
to the user, sorted by the timestamps of the events.

As the output of an event camera is sparse and very different from
the one of a frame-based camera, it requires development of novel
approaches in order to exploit its properties and potential. 

One common approach that can be found in the literature, is to take
the special output of the camera, make artificial frames out of it,
and apply classical algorithms based on the frame-baesd paradigm.
However, this approach, conceptually, will lead to bad results in
a few aspects, as delicate temporal information gets lost when the
frames are made out of the events list, and artificial unnecessary
information of zero values is added to the frames. 

Instead, as mentioned in the introduction, one needs to develop event-based
asynchronous algorithms to process the output. In the following chapter
we are going to present such an algorithm, with the goal of periodic
signals detection.

\section{Description of the algorithm}

As was stated in the introduction, we consider only the positive events
of the output. However, depending on the application and context,
one can take the negative events instead, or consider both polarities.
Therefore, the input of the algorithm is the list of positive events
generated by the event camera 
\[
L=\{(t_{i},x_{i},y_{i})\,\,|\,\,i=1,\ldots,k\}
\]
where $t_{i}$ is the $i$-th event timestamp, $(x_{i},y_{i})$ is
its pixel coordinates, and $k$ is the number of events. We define
the two following variables, which we keep constant along the algorithm: 
\begin{enumerate}
\item $\delta:=$ the period of the signal we are looking for. 
\item $\epsilon:=$ the error we consider in the period of the signal. Typical
value for this variable should be the expected rise time of the signal. 
\end{enumerate}
let $u,v$ be the dimensions of the event camera pixel array. As mentioned
in the introduction, we use the notion of \textquotedbl time surface\textquotedbl .
Namely, in the initialization of the algorithm, we define a $2D$
array of dimensions $u\times v$ which we initialize its entries arbitrarily
to be some negative number, smaller than $-\delta-\epsilon$. We denote
it by $TS$, and we update it along the algorithm to store the last
timestamp of a positive event reported by each of the pixels. In addition,
we define two $2D$ arrays of dimensions $u\times v$, initialized
to zero, which we denote by $Total$ and $Periodic$. Now, using a
for-loop going along $i=1,\ldots,k$, the algorithm implements 4 steps.
Here are the first 3 of them: 
\begin{enumerate}
\item $Total_{x_{i},y_{i}}:=Total_{x_{i},y_{i}}+1$. 
\item If $|t_{i}-TS_{x_{i},y_{i}}-\delta|<\epsilon$, then $Periodic_{x_{i},y_{i}}:=Periodic_{x_{i},y_{i}}+1$. 
\item $TS_{x_{i},y_{i}}:=t_{i}$. 
\end{enumerate}
In order to present the final step, consider a certain pixel $(x,y)$
and denote 
\[
m=Total_{x,y},\qquad n=Periodic_{x,y}
\]
Now, if our recording is $T$ seconds long, and the events reported
by the pixel are due to the desired periodic signal, we expect $n$
to get close to $\frac{T}{\delta}$ during the for-loop, or at least
to be big enough so it will be unlikely to relate it to a random flickering
signal. How big should it be? This is exactly what we are analyzing
below. To formulate the problem, let $0\leq\tau_{1}\leq\tau_{2}\le\cdots\leq\tau_{m}\leq T$
originated from an unknown distribution. For $d\in\mathbb{N}$, denote
the probability 
\begin{equation}
\Phi(m,d)=P(\#\{j\,\,|\,\,|\tau_{j+1}-\tau_{j}-\delta|<\epsilon\}\geq d)\label{eq:Phi}
\end{equation}
Then, in general, $\Phi$ is monotonically decreasing as a function
of $d$, $\Phi(m,0)=1$, and $\Phi(m,m)=0$. Now, let $q$ be the
allowed probability for false alarm in one pixel. We note that in
general, $q$ should be much smaller than the allowed probability
for false alarm in the whole field of view. Then, there exists a minimal
$D$ such that $\Phi(m,D)\leq q$. Now, let $0\leq s_{1}\leq s_{2}\leq\cdots\leq s_{m}\leq T$
be the timestamps of the events reported by the pixel $(x,y)$. Then
\[
n=Periodic_{x,y}=\#\{j\,\,|\,\,|s_{j+1}-s_{j}-\delta|<\epsilon\}.
\]
We would like to say that the signal reported by the pixel $(x,y)$
is unlikely to relate to a random flickering if $n\geq D$, which
is equivalent to the condition 
\[
\Phi(m,n)\leq q.
\]

Denote another $2D$ array of dimensions $u\times v$, initialized
to zero, by $Detected$, to indicate, for each pixel, whether, considering
the information processed by the for-loop so far, the pixel detected
the signal. Now, we can present the final step of the for-loop:
\begin{itemize}
\item If $\Phi(Total_{x,y},Periodic_{x,y})\leq q$ and $Detected_{x,y}=0$,
then add the pixel $(x,y)$ to the list of detections and update the
value of $Detected_{x,y}$ to $1$.
\item If $\Phi(Total_{x,y},Periodic_{x,y})>q$ and $Detected_{x,y}=1$,
then remove the pixel $(x,y)$ from the list of detections and update
the value of $Detected_{x,y}$ to $0$.
\end{itemize}
We emphasis that this final step can be written as a forth step in
the for-loop itself and does not need to come afterwards. This way,
the algorithm does not need to look for the signal in each pixel after
the for-loop is done: the detections are found incidentally trough
the temporal event loop. The output of the algorithm is then a list
of pixels in which a periodic signal of the desired frequency has
been detected, considering all the events in the pixels during the
$T$-seconds recording.

We do not know about accurate estimations of $\Phi(m,d)$ in the literature,
even in simple cases, e.g. when the events are uniformly distribution
originated. However, in the following we suggest a practical way to
approximate its values in order to complete the decision making of
the algorithm. 

\section{Suggested analysis}

Assume that $0\leq\tau_{1}\leq\tau_{2}\le\cdots\leq\tau_{m}\leq T$
are unknown distribution originated. We start with estimating the
probability for the distance between two arbitrary subsequent $\tau_{i}$-s
to be close to $\delta$ up to an error $\epsilon$. For this estimation,
we assume that the appearance of the $\tau_{i}$-s is exponentially
distributed with a parameter $\lambda$. Under this assumption, and
the assumption $\epsilon\leq\delta$, the probability we are looking
for is given by the formula
\[
\tilde{P}(\lambda)=\int_{\delta-\epsilon}^{\delta+\epsilon}\lambda e^{-\lambda s}\cdot ds=(e^{\lambda\epsilon}-e^{-\lambda\epsilon})e^{-\lambda\delta}.
\]
However, as we do not know the value of $\lambda$, but we do know
that between $0$ and $T$ we have $m$ $\tau_{i}$-s, we are weighting
this probability by the probability of having $m$ events given a
Poisson distribution with parameter $\rho=\lambda T$. Hence, we estimate
the desired probability by the formula 
\begin{align*}
P & =\text{\ensuremath{\int}}_{0}^{\text{\ensuremath{\infty}}}\tilde{P}(\lambda)\cdot\frac{\rho^{m}e^{-\rho}}{m!}\cdot d\rho\\
 & =T\text{\ensuremath{\int}}_{0}^{\text{\ensuremath{\infty}}}\frac{(\lambda T)^{m}(e^{-\lambda(T+\delta-\epsilon)}-e^{-\lambda(T+\delta+\epsilon)})}{m!}\cdot d\lambda\\
 & =T^{m+1}\left(\frac{1}{(T+\delta-\epsilon)^{m+1}}-\frac{1}{(T+\delta+\epsilon)^{m+1}}\right).
\end{align*}
where the latter equality is yields by using repeatedly the method
of integration by parts. 
\begin{rem}
Expending $P$ and $\tilde{P}$, one can see that if $\delta\ll T$,
then $\tilde{P}\left(\frac{m+1}{T}\right)$ gives a good approximation
for $P$. This is not far from the intuitive sense that in high probability,
the parameter $\lambda$ is close to the value $\lambda=\frac{m}{T}$.
Indeed, under the aforementioned assumptions we have

\begin{align*}
P & =\frac{1}{\left(1+\frac{\delta-\epsilon}{T}\right)^{m+1}}-\frac{1}{\left(1+\frac{\delta+\epsilon}{T}\right)^{m+1}}\\
 & =\sum_{k=0}^{\infty}\frac{(m+k)!}{(k!\cdot m!)}\left(\left(\frac{-\delta+\epsilon}{T}\right)^{k}-\left(\frac{-\delta-\epsilon}{T}\right)^{k}\right)\\
 & \approx(m+1)\cdot\frac{2\epsilon}{T}+(m+2)(m+1)\cdot\frac{2\epsilon\delta}{T^{2}}\\
\\
\tilde{P}\left(\frac{m+1}{T}\right) & =e^{\frac{m+1}{T}(\epsilon-\delta)}-e^{\frac{m+1}{T}(-\epsilon-\delta)}\\
 & =\sum_{k=0}^{\infty}\frac{(m+1)^{k}}{k!}\left(\left(\frac{-\delta+\epsilon}{T}\right)^{k}-\left(\frac{-\delta-\epsilon}{T}\right)^{k}\right)\\
 & \approx(m+1)\cdot\frac{2\epsilon}{T}+(m+1)^{2}\cdot\frac{2\epsilon\delta}{T^{2}}.
\end{align*}
\end{rem}

Notice now that given the information of having $m$ events reported
by the pixel, the probability for two adjacent events of the pixel
to be close to each other in the appropriate distance is not independent.
However, for the sake of simplicity we do not take into account this
fact. Hence, the probability of at least $d$ adjacent events out
of $M=m-1$ to be close to each other in the appropriate distance
can be approximated by 
\[
Q(m,d)=\sum_{l=d}^{m-1}\binom{m-1}{l}P(m)^{l}(1-P(m))^{m-1-l}.
\]
In the sequent, we use $Q(m,d)$ as an approximation for the desired
$\Phi(m,d)$ presented in Equation \ref{eq:Phi}. 

It is obvious that just like $\Phi(m,d)$, for any $m$, $Q(m,d)$
is monotonically decreasing as function of $d$. However, the dependence
of $Q(m,d)$ on $m$ is more tricky as the parameter $m$ influences
$Q$ in two opposite directions. From one hand, when $m$ grows it
means that there are more events, and hence it should be easier to
reach the threshold $d$, so $Q$ should grow. On the other hand,
when $m$ grows, the event rate grows as well, and hence the chance
for a specific random event to be far from the previous one in the
right interval gets smaller, so $P$ decreases, and hence $Q$ should
decrease also. This observation is summarized in the following proposition,
showing that the function $Q$ does encode these two opposite effects. 
\begin{prop}
\label{prop:lim}Given a specific value for the variables $\delta$,
$\epsilon$, $T$, $d$, with $\epsilon\leq\frac{\delta}{2}$ and
$d>0$, one has:

\[
\lim_{m\to0}Q(m,d)=\lim_{m\to\infty}Q(m,d)=0.
\]
In particular, Q reaches a maximum value as a function of m. 
\end{prop}

\begin{proof}
Writing $P=P(m)$ and using the assumption $d>0$, one has
\begin{align*}
Q(m,d) & =\sum_{l=d}^{m-1}\left(\begin{array}{c}
m-1\\
l
\end{array}\right)P^{l}(1-P)^{m-1-l}\\
 & =P\cdot\sum_{l=d}^{m-1}\left(\begin{array}{c}
m-1\\
l
\end{array}\right)P^{l-1}(1-P)^{m-1-l}\\
 & =P\cdot\sum_{l=d-1}^{m-2}\frac{m-1}{l+1}\left(\begin{array}{c}
m-2\\
l
\end{array}\right)P^{l}(1-P)^{m-2-l}\\
 & \leq P\cdot m\cdot\sum_{l=d-1}^{m-2}\left(\begin{array}{c}
m-2\\
l
\end{array}\right)P^{l}(1-P)^{m-2-l}\\
 & =P\cdot m\cdot Q(m-1,d-1)\leq P\cdot m\\
 & =m\cdot T^{m+1}\left(\frac{1}{(T+\delta-\epsilon)^{m+1}}-\frac{1}{(T+\delta+\epsilon)^{m+1}}\right)\overset{m\to0}{\longrightarrow}0.\\
\end{align*}
As we assume that $\epsilon\leq\frac{\delta}{2}$, we have $\delta-\epsilon>0$,
and hence we also have
\begin{align*}
Q(m,d) & \leq\frac{m\cdot T^{m+1}}{(T+\delta-\epsilon)^{m+1}}\\
 & =\frac{m}{\left(1+\frac{\delta-\epsilon}{T}\right)^{m+1}}\overset{m\to\infty}{\longrightarrow}0.
\end{align*}
\end{proof}

\section{Demonstration}

In general, the algorithm has better performance with detecting fast
rising periodic signals, as in these cases one can choose the error
parameter $\epsilon$ of the algorithm to be relatively small. However,
in order to feel its effectiveness lower bound, we demonstrate it
here on the periodic signal of a streetlight powered by sinusoidal
current of the electrical grid. In this case, the rising time of the
signal is basically half of the period. Hence, if we do not want to
miss any of the events related to the periodic signal, we need to
choose $\epsilon$ to its highest reasonable value, namely $\epsilon=\frac{\delta}{2}$.
In this case, as the frequency of the signal is $100$ Hz, we have
$\delta=10$ ms, and hence $\epsilon=5$ ms. 

Our experimental setup is built up of two cameras staring on approximately
the same urban view: Prophesee Gen4 event-camera with resolution of
1M pixel, and a CMOS frame camera for guidance. Figure \ref{fig:picture}
shows the intersection of the two cameras field of view, seen by the
frame camera. This intersection covers the field of view of approximately
$1000\times600$ pixels of the event camera. 

In Figure \ref{fig:picture} we point out some well seen sun glittering
coming from the buildings on the mountains, but there are also some
intensive pixel size challenging sun glittering signals from the closer
urban view that cannot be easily seen in the picture. We also point
out the location of a streetlight. 

\begin{figure}[H]
\begin{centering}
\includegraphics[width=3.5in]{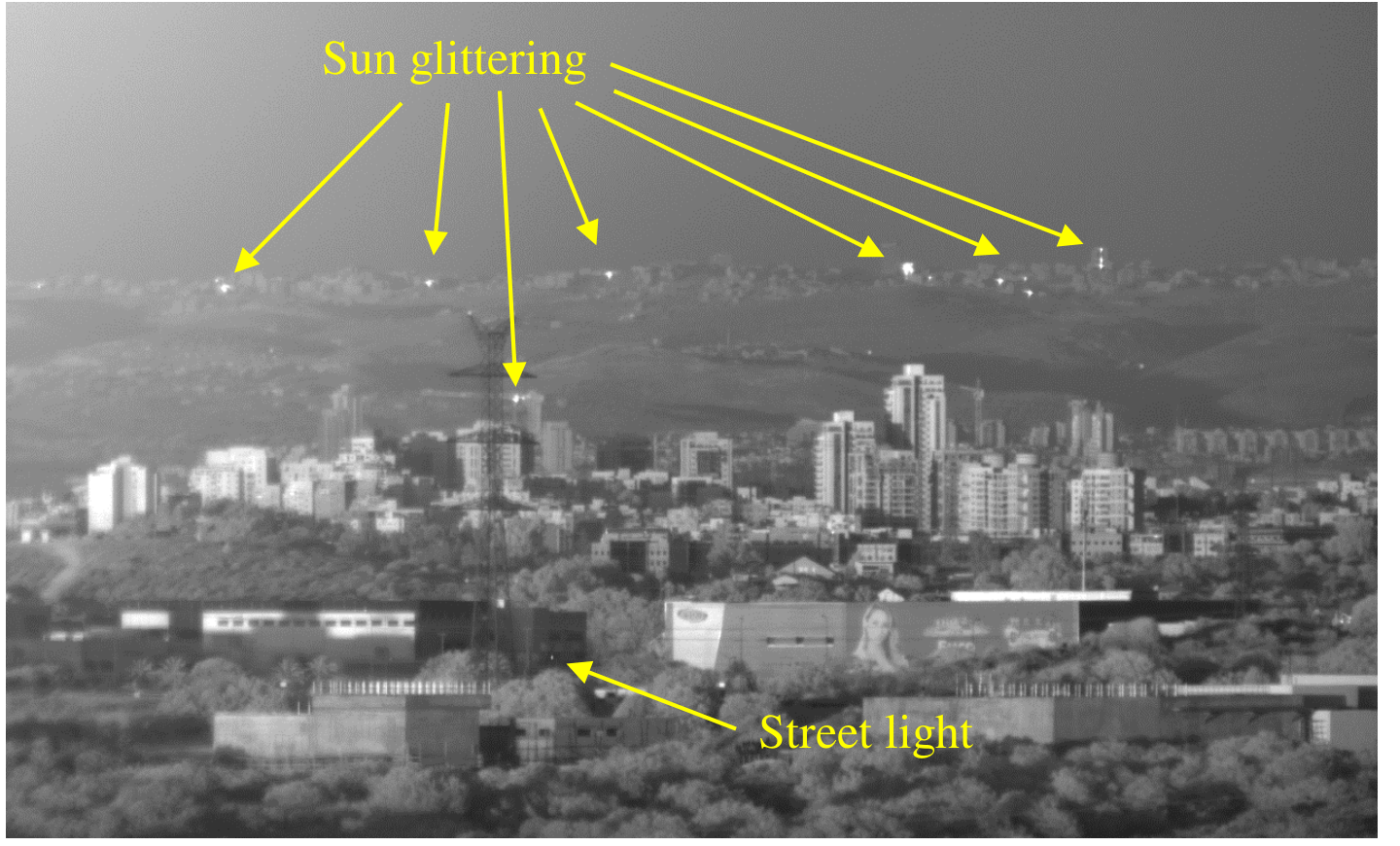}
\par\end{centering}
\caption{\label{fig:picture}The intersection of the frame camera and event
camera field of view. }
\end{figure}

Now, in order to demonstrate how the algorithm distinguishes between
the periodic signal of the streetlight, and all the other events,
we recorded one second of the scene, i.e. $T=1\,\,s$. Then, we applied
the algorithm on the recording, and evaluated the function $Q(m,n)$
for each pixel. 

The logarithmic graph in Figure \ref{fig:graph} shows the number
of pixels in the intersection field of view that reported more than
5 events during the recording, for which the probability function
$Q(m,n)$ is smaller than the value of the allowed probability for
false alarm given in the horizontal axis. In addition, it shows how
many pixels out of them, are related to the periodic signal of the
streetlight. 

\begin{figure}[H]
\centering{}\includegraphics[width=3.5in]{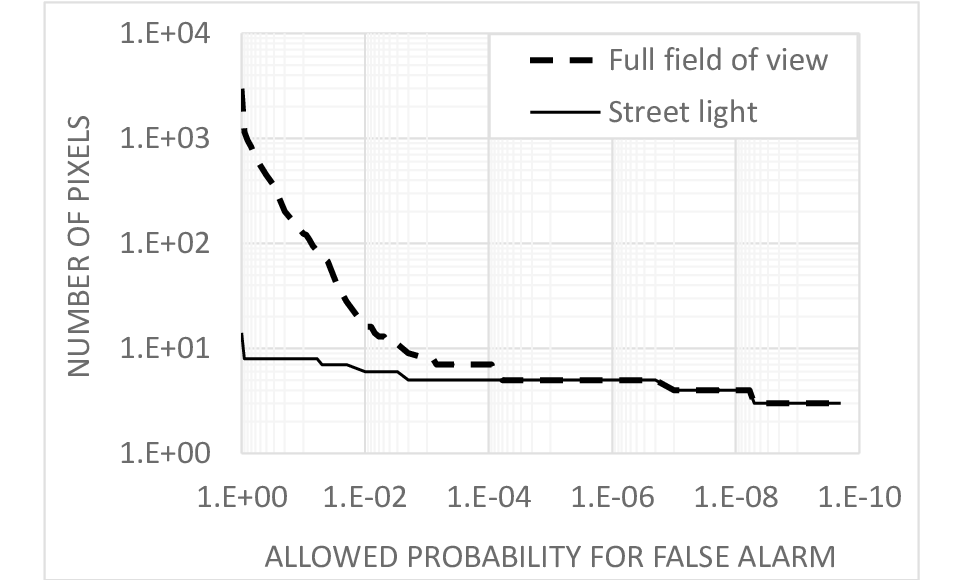}\caption{\label{fig:graph}Number of pixels with probability function value
smaller than the allowed probability for false alarm. }
\end{figure}

The graph shows that when the probability for false alarm is smaller
than $10^{-5}$, we remain with only 5 pixels that have enough events
in the output of the algorithm, and all of them are related to the
streetlight. Moreover, the value of the probability function of these
5 pixels is at least 2 orders of magnitude less than $10^{-5}$. In
other words, there is a very clear dichotomy between the probability
function of these 5 pixels and the one of the other pixels. Hence,
these 5 pixels are very well distinguished by the algorithm and the
analysis presented here.

One can also see in Figure \ref{fig:graph} that there were additional
9 pixels that have reported more than 5 events during the recording,
and their events are related to the streetlight, but were flickering
more weakly, and hence the algorithm could not distinguish them effectively
from other random signals. In Table 1 we give the values of the variables
that were involved in evaluating the probability function of all the
pixels that responded to the flickering street light: the 5 distinguished
pixels and the other 9 with weak signal, sorted according to the value
of the probability function. 

\begin{table}[H]
\begin{centering}
\begin{tabular}{|c|c|c|c|}
\hline 
Pixel & $m$ & $n$ & $Q$\tabularnewline
\hline 
\hline 
1 & 554 & 99 & $\sim10^{-16}$\tabularnewline
\hline 
2 & 506 & 98 & $\sim2\cdot10^{-16}$\tabularnewline
\hline 
3 & 160 & 98 & $\sim3\cdot10^{-11}$\tabularnewline
\hline 
4 & 376 & 99 & $\sim6\cdot10^{-9}$\tabularnewline
\hline 
5 & 342 & 99 & $\sim10^{-7}$\tabularnewline
\hline 
6 & 132 & 66 & $\sim2\cdot10^{-3}$\tabularnewline
\hline 
7 & 188 & 76 & $\sim2\cdot10^{-2}$\tabularnewline
\hline 
8 & 89 & 41 & $\sim5\cdot10^{-2}$\tabularnewline
\hline 
9 & 26 & 2 & $\sim1$\tabularnewline
\hline 
10 & 80 & 20 & $\sim1$\tabularnewline
\hline 
11 & 96 & 24 & $\sim1$\tabularnewline
\hline 
12 & 22 & 0 & $\sim1$\tabularnewline
\hline 
13 & 10 & 0 & $\sim1$\tabularnewline
\hline 
14 & 9 & 0 & $\sim1$\tabularnewline
\hline 
\end{tabular}
\par\end{centering}
\caption{Variables of the pixels reacting to the signal of the streetlight.}

\end{table}

\begin{rem}
By analyzing the periodic signals more carefully, one can find a better
value for the error parameter $\epsilon$, say $\epsilon=3$ ms, and
be able to add pixels 6-8 in Table 1 to the list of the distinguished
pixels. We also remark that obviously, analyzing a longer recording
can help to detect these pixels as well.
\end{rem}

\begin{rem}
Concentrating on the $5$ distinguished pixels, one can see that for
all of them we have $98\leq n\leq99$ which is very close to the expected
number $\frac{T}{\delta}=100$, but $m$ varies between $160\leq m\leq554$.
The reason $m$ is much higher than the expected number of events
in the pixel, is that in each of the periods, the rise of the signal
triggered more than one event: as powerful as the signal was, more
events were triggered in each of the periods. Now, focusing on the
relation between $m$ and $Q$, one can see that $Q$ reached a higher
value when $m=342$ (Pixel 5) than in the case $m=160$ (Pixel 3).
But after that, when we move from Pixel $5$ to Pixel $4$ and so
on, when $m$ grows, $Q$ becomes smaller. This phenomenon demonstrates
the validation of Proposition \ref{prop:lim}.\\
\\
Ben-Ezra, David El-Chai, Remote Sensing Department, Soreq NRC, Yavne,
Israel, 81800 (dbenezra@mail.huji.ac.il)\\
\\
Arad, Ron, Remote Sensing Department, Soreq NRC, Yavne, Israel, 81800
(fnarad@yahoo.com)\\
\\
Padowicz, Ayelet, Remote Sensing Department, Soreq NRC, Yavne, Israel,
81800 (ayeletp@soreq.gov.il)\\
\\
Tugendhaft, Israel, Remote Sensing Department, Soreq NRC, Yavne, Israel,
81800 (tugen@soreq.gov.il) 
\end{rem}

\end{document}